\documentclass[11pt,a4paper,final]{article}
\usepackage[utf8]{inputenc}
\usepackage{amsmath}
\usepackage{amsfonts}
\usepackage{amssymb}
\usepackage{graphicx}
\usepackage[left=2cm,right=2cm,top=2cm,bottom=2cm]{geometry}

\usepackage{color}
\newcommand{\hl}[1]{\textcolor{black}{#1}}

\usepackage{amsmath}
\usepackage{natbib}
\usepackage{amsfonts}
\usepackage{amssymb}
\usepackage{graphicx}
\usepackage{subcaption}
\usepackage{algorithm2e}
\usepackage{multirow}
\usepackage{wrapfig}
\usepackage{url}

\newtheorem{proposition}{Proposition}
\newtheorem{definition}{Definition}

\newcommand{\Exp}{\mathbb{E}}
\newcommand{\PR}{\mathbb{P}}

\newcommand{\reals}{\mathbb{R}}

\newcommand{\vw}{\mathbf{w}}
\newcommand{\vx}{\mathbf{x}}

\newcommand{\vz}{\mathbf{z}}

\newcommand{\hy}{\hat{y}}
\newcommand{\hY}{\hat{Y}}

\newcommand{\vtheta}{\boldsymbol{\theta}}

\newcommand{\vX}{\mathbf{X}}
\newcommand{\vZ}{\mathbf{Z}}
\newcommand{\vW}{\mathbf{W}}

\newcommand{\cD}{\mathcal{D}}

\newcommand{\cH}{\mathcal{H}}

\newcommand{\cY}{\mathcal{Y}}
\newcommand{\cX}{\mathcal{X}}
\newcommand{\cZ}{\mathcal{Z}}
\newcommand{\cW}{\mathcal{W}}

\newcommand{\cF}{\mathcal{F}}
\newcommand{\cL}{\mathcal{L}}

\newcommand{\qed}{\hfill$\rule{2mm}{3mm}$}
\newenvironment{proof}{\par{\noindent \bf Proof }}{\qed \par}

\begin{document}
\title{A Moral Framework for Understanding of Fair ML\\through Economic Models of Equality of Opportunity}

\author{
   \makebox[.34\linewidth]{Hoda Heidari }\\
   ETH Z{\"u}rich\\
   \url{hheidari@inf.ethz.ch} \\
   \and
   \makebox[.34\linewidth]{Michele Loi} \\
   University of Z{\"u}rich\\
   \url{michele.loi@uzh.ch} \\
   \and
   \makebox[.34\linewidth]{Krishna P. Gummadi} \\
   MPI-SWS\\
   \url{gummadi@mpi-sws.org} \\
   \and
   \makebox[.34\linewidth]{Andreas Krause}\\
   ETH Z{\"u}rich\\
   \url{krausea@ethz.ch} \\
}

\date{}

\maketitle

\begin{abstract}
We map the recently proposed notions of algorithmic fairness to economic models of Equality of opportunity (EOP)---an extensively studied ideal of fairness in political philosophy. We formally show that through our conceptual mapping, many existing definition of algorithmic fairness, such as predictive value parity and equality of odds, can be interpreted as special cases of EOP.
In this respect, our work serves as a unifying moral framework for understanding existing notions of algorithmic fairness. Most importantly, this framework allows us to explicitly spell out the moral assumptions underlying each notion of fairness, and interpret recent fairness impossibility results in a new light. Last but not least and inspired by luck egalitarian models of EOP, we propose a new family of measures for algorithmic fairness. We illustrate our proposal empirically and show that employing a measure of algorithmic (un)fairness when its underlying moral assumptions are not satisfied, can have devastating consequences for the disadvantaged group's welfare. 
\end{abstract}

\section{Introduction}
\emph{Equality of opportunity} (EOP) is a widely supported ideal of fairness, and it has been extensively studied in political philosophy over the past 50 years~\citep{rawls2009theory,sen1980equality,dworkin1981equality1,dworkin1981equality2,arneson1989equality,cohen1989currency}. The concept assumes the existence of a broad range of \emph{positions}, some of which are more desirable than others. In contrast to \emph{equality of outcomes} (or positions), an equal opportunity policy seeks to create a \emph{level playing field} among individuals, after which they are free to compete for different positions. The positions that individuals earn under the condition of equality of opportunity reflect their \emph{merit} or \emph{deservingness}, and for that reason, inequality in outcomes is considered ethically acceptable~\citep{roemer2002equality}. 

Equality of opportunity emphasizes the importance of personal (or native) qualifications, and seeks to minimize the impact of circumstances and arbitrary factors on individual outcomes~\citep{cohen1989currency,dworkin1981equality1,dworkin1981equality2,rawls2009theory}. For instance within the context of employment, one (narrow) interpretation of EOP requires that desirable jobs are given to those persons most likely to perform well in them---e.g. those with the necessary education and experience---and not according to arbitrary factors, such as race or family background.
According to \citeauthor{rawls2009theory}'s (broader) interpretation of EOP, native talent and ambition can justify inequality in social positions, whereas circumstances of birth and upbringing such as sex, race, and social background can not.
Many consider the distinction between morally acceptable and unacceptable inequality the most significant contribution of the egalitarian doctrine~\citep{roemer2015equality}.

Prior work in economics has sought to formally characterize conditions of equality of opportunity to allow for its precise measurement in practical domains (see e.g.~\citep{fleurbaey2008fairness,roemer2009equality}). At a high level, in these models an individual's outcome/position is assumed to be affected by two main factors: his/her \emph{circumstance} $c$ and \emph{effort} $e$. Circumstance $c$ is meant to capture all factors that are deemed irrelevant, or for which the individual should not be held morally \emph{accountable}; for instance $c$ could specify the socio-economic status he/she is born into. Effort $e$ captures all accountability 
 factors---those that can morally justify inequality. (Prior work in economics refers to $e$ as effort for the sake of concreteness, but $e$ summarizes \emph{all} factors for which the individual can be held morally accountable; the term ``effort" should not be interpreted in its ordinary sense here.) For any circumstance $c$ and any effort level $e$, a policy $\phi$ induces a distribution of \emph{utility} among people of circumstance $c$ and effort $e$. Formally, an EOP policy will ensure that an individual's final utility will be, to the extent possible, only a function of their effort and not their circumstances.

While EOP has been traditionally discussed in the context of employment practices, its scope has been expanded over time to other areas, including lending, housing, college admissions, and beyond~\citep{wiki-EOP}. Decisions made in such domains are increasingly automated and made through Algorithmic Data Driven Decision Making systems (A3DMs). We argue, therefore, that it is only natural to study fairness for A3DMs through the lens of EOP. In this work, we draw a formal connection between the recently proposed notions of fairness for supervised learning and economic models of EOP. We observe that in practice, predictive models inevitably make errors (e.g. the model may mistakenly predict that a credit-worthy applicant won't pay back their loan in time). Sometimes these errors are beneficial to the subject, and sometimes they cause harm. We posit that in this context, EOP would require similar individuals (in terms of what they can be held accountable for) to have the same prospect of receiving this benefit/harm, irrespective of their irrelevant characteristics. 

More precisely, we assume that a person's features can be partitioned into two sets: those for which we consider it morally acceptable to hold him/her accountable, and those for which it is not so. We will broadly refer to the former set of attributes as the individual's \emph{accountability} features, and the latter, as their \emph{arbitrary} or \emph{irrelevant} features. Note that there is considerable disagreement on the criteria to determine what factors should belong to each category. \citet{roemer1993pragmatic} for instance proposes that societies decide this democratically. We take a neutral stance on this issue and leave it to domain experts and stake-holders to reach a resolution. Throughout, we assume this partition has been identified and is given.

We distinguish between an individual's \emph{actual} and \emph{effort-based} utility when subjected to algorithmic decision making. We assume an individual's \emph{advantage} or \emph{total utility} as the result of being subject to A3DMs, is the difference between their actual and effort-based utility (Section~\ref{sec:model}). 
Our main conceptual contribution is to map the supervised learning setting to economic models of EOP by treating predictive models as policies, irrelevant features as individual circumstance, and effort-based utilities as effort (Figure~\ref{fig:map}). We show that using this mapping many existing notions of fairness for classification, such as predictive value parity~\citep{kleinberg2016inherent} and equality of odds~\citep{hardt2016equality}, can be interpreted as special cases of EOP. In particular, equality of odds is equivalent to Rawlsian EOP, if we assume all individuals with the same true label are equally accountable for their labels and have the same effort-based utility (Section \ref{sec:rawlsian}). Similarly, predictive value parity is equivalent to luck egalitarian EOP if the predicted label/risk is assumed to reflect an individual's effort-based utility (Section \ref{sec:egalitarian}).
In this respect, our work serves as a unifying framework for understanding existing notions of algorithmic fairness as special cases of EOP. Importantly, this framework allows us to explicitly spell out the moral assumptions underlying each notion of fairness, and interpret recent fairness impossibility results~\citep{kleinberg2016inherent} in a new light.
\begin{figure}
  \centering
    \includegraphics[width=0.48\textwidth]{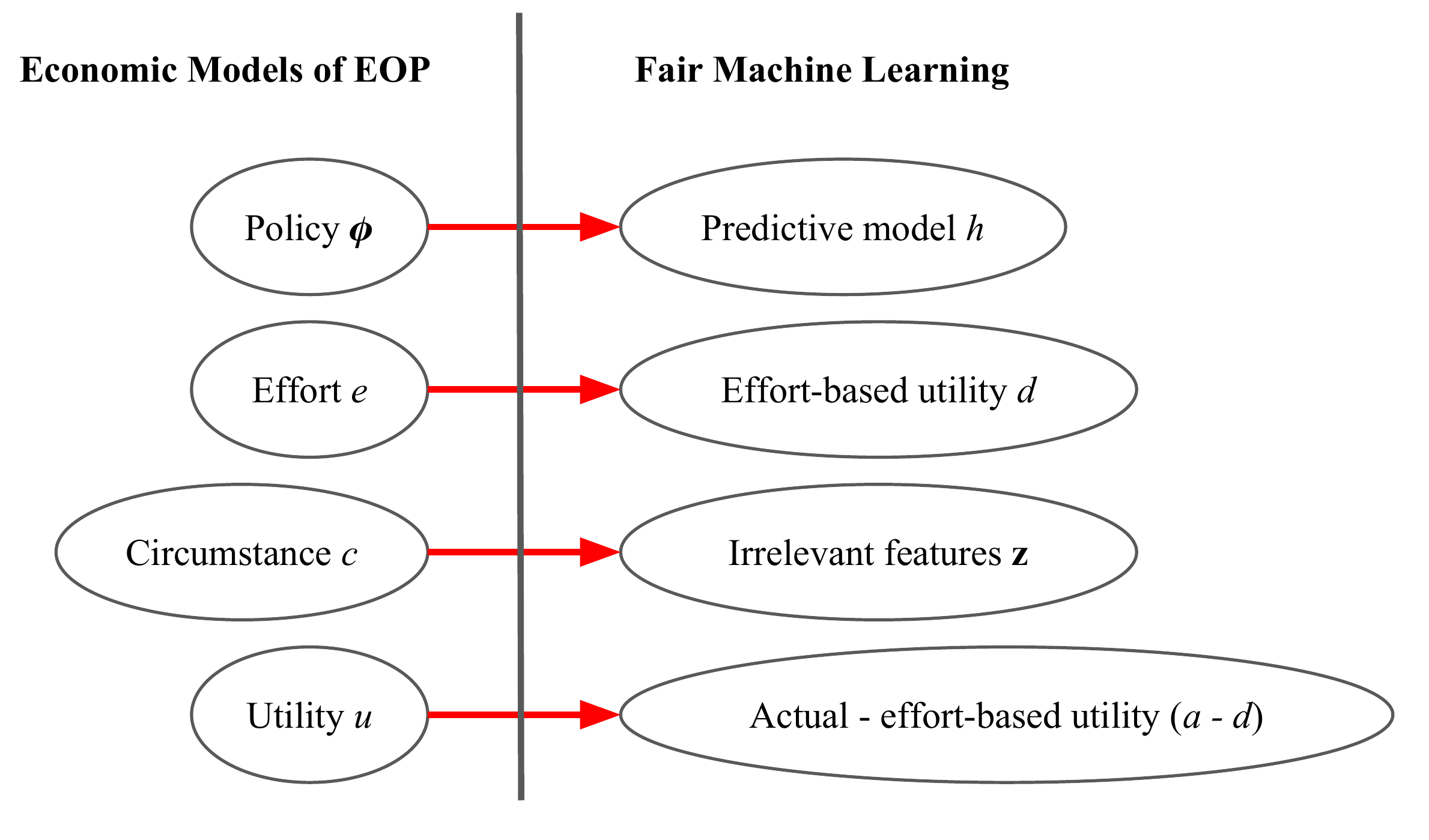}
    \caption{Our proposed conceptual mapping between Fair ML and economic literature on EOP.}\label{fig:map}
\end{figure}

Last but not least, inspired by Roemer's model of egalitarian EOP we present a new family of measures for algorithmic (un)fairness, applicable to supervised learning tasks beyond binary classification. We illustrate our proposal on a real-world regression dataset, and compare it with existing notions of fairness for regression. We  empirically show that employing the wrong measure of algorithmic fairness---when the moral assumptions underlying it are not satisfied---can have devastating consequences on the welfare of the disadvantaged group.

\hl{
We emphasize that our work is \emph{not} meant to advocate for any particular notion of algorithmic fairness, rather our goal is to establish---both formally and via real-world examples---that implicit in each notion of fairness is a distinct set of moral assumptions about decision subjects; therefore, each notion of fairness is suitable only in certain application domains and not others. By making these assumptions explicit, our framework presents practitioners with a \emph{normative} guideline to choose the most suitable notion of fairness specifically for every real-world context in which A3DMs are to be deployed.
}

\subsection{Equality of Opportunity: An Overview}\label{sec:related}
Equality of opportunity has been extensively debated among political philosophers. 
Philosophers such as \citet{rawls2009theory}, \citet{dworkin1981equality1}, \citet{arneson1989equality}, and \citet{cohen1989currency} contributed to the egalitarian school of thought by proposing different criteria for making the cut between arbitrary and accountability factors. The detailed discussion of their influential ideas is outside the scope of this work, and the interested reader is referred to excellent surveys by~\citet{sep-equal-opportunity} and \citet{roemer2015equality}.

%
In this section, we briefly mention several prominent interpretations of EOP and discuss their relevance to A3DMs. 
Following \citet{arneson2018four}, we recount three main conceptions of equality of opportunity:
\begin{itemize}
\item \textbf{Libertarian EOP:} A person is morally at liberty to do what she pleases with what she legitimately owns (e.g. self, business, etc.) as long as it does not infringe upon other people's moral rights (e.g. the use of force, fraud, theft, or damage on persons or property of another individual is considered a violation of their rights). Other than these restrictions, any outcome that occurs as the result of people's free choices on their legitimate possessions is considered just. In the context of A3DMs and assuming no gross violations of individuals' data privacy rights, this interpretation of EOP leaves the enterprise at total liberty to implement any algorithm it wishes for decision making. The algorithm can utilize all available information, including individuals' sensitive features such as race or gender, to make (statistically) accurate predictions. 

\item \textbf{Formal EOP:} Also known as ``careers open to talents", formal EOP require desirable social positions to be open to all who possess the attributes relevant for the performance of the duties of the position (e.g. anyone who meets the formal requirements of the job) and wish to apply for them~\citep{roemer2009equality}. The applications must be assessed only based on relevant attributes/qualifications that advances the morally innocent goals of the enterprise. Direct discrimination based on factors deemed arbitrary (e.g. race or gender) is therefore prohibited under this interpretation of EOP. 
Formal EOP would permit differences in people's circumstances---e.g. their gender---to have indirect, but nonetheless deep impact on their prospects. For instance, if women are less likely to receive higher education due to prejudice against female students, as long as a hiring algorithm is blind to gender and applies the same educational requirement to male and female job applicants, formal equality of opportunity is maintained. In context of A3DMs, Formal EOP is equivalent to the removal of the sensitive feature information from the learning pipeline. In the fair ML community, this is sometimes referred to as ``fairness through blindness".

\item \textbf{Substantive EOP:} 
Substantive EOP moves the starting point of the competition for desirable positions further back in time, and requires not only open competition for desirable positions, but also fair access to the necessary qualifications for the position. This implies access to qualifications (e.g. formal requirements for a job) should not to be affected by arbitrary factors, such as race gender or social class. 
The concept is closely related to \emph{indirect discrimination}: if the A3DM indirectly discriminates against people with a certain irrelevant feature (e.g. women or African Americans) this may be an indication that the irrelevant/arbitrary feature has played a role in the acquisition of the requirements. When there are no alternative morally acceptable explanations for it, indirect discrimination is often considered in violation of substantive EOP. 
\end{itemize}
Our focus in this work is on substantive EOP, and in particular, on two of its refinements, called Rawlsian EOP and Luck Egalitarian EOP.

\paragraph{\textbf{Rawlsian EOP}} According to Rawls, those who have the same level of talent or ability and are equally willing to use them must have the same \emph{prospect} of obtaining desirable social positions, regardless of arbitrary factors such as socio-economic background~\citep{rawls2009theory}. This Rawlsian conception of EOP has been translated into precise mathematical terms as follows~\citep{lefranc2009equality}:
let $c$ denote circumstance, capturing factors that are not considered legitimate sources of inequality among individuals. Let scalar $e$ summarize factors that are viewed as legitimate sources of inequality. For the sake of brevity, the economic literature refer to $e$ as ``effort", but $e$ is meant to summarize all factors an individual can be held morally accountable for.\footnote{Note that in Rawls's formulation of EOP, talent and ambition are treated as a legitimate source of inequality, even when they are independent of a person's effort and responsibility. The mathematical formulation proposed here includes talent, ability and ambition all in the scalar $e$. Whether natural talent should be treated as a legitimate source of inequality is a subject of controversy. As stated earlier, throughout this work we assume such questions have been already answered through a democratic process and/or deliberation among stakeholders and domain experts.} Let $u$ specify individual utility, which is a consequence of effort, circumstance, and policy. Formally, let $F^\phi(. \vert c,e)$ specify the cumulative distribution of utility under policy $\phi$ at a fixed effort level $e$ and circumstance $c$. 
Rawlsian/Fair EOP requires that for individuals with similar effort $e$, the distribution of utility should be the same---regardless of their circumstances:
\begin{definition}[Rawlsian Equality of Opportunity (R-EOP)]
A policy $\phi$ satisfies Rawlsian EOP if for all circumstances $c,c'$ and all effort levels $e$, $$F^\phi(. \vert c,e) = F^\phi(. \vert c',e).$$
\end{definition}
Note that this conception of EOP takes an \emph{absolutist} view of effort: it assumes $e$ is a scalar whose absolute value is meaningful and can be compared across individuals. This view requires effort $e$ to be inherent to individuals and not itself impacted by the circumstance $c$ or the policy $\phi$. 

\paragraph{\textbf{Luck Egalitarian EOP}} Unlike fair EOP, luck egalitarian EOP offers a \emph{relative} view of effort, and allows for the possibility of circumstance $c$ and implemented policy $\phi$ impacting the distribution of effort $e$. In this setting, \citet{roemer2002equality} argues that ``in comparing efforts of individuals in different types [circumstances], we should somehow adjust for the fact that those efforts are drawn from distributions which are different". As the solution he goes on to propose ``measuring a person's effort by his rank in the effort distribution of his type/circumstance, rather than by the absolute level of effort he expends".

Formally, let $F^{c, \phi}_E$ be the effort distribution of type $c$ under policy $\phi$. Roemer argues that ``this distribution is a characteristic of the type $c$, not of any individual belonging to the type. Therefore, an inter-type comparable measure of effort must factor out the goodness or badness of this distribution". Roemer declares two individuals as having exercised the same level of effort if they sit at the same quantile or rank of the effort distribution for their corresponding types. 
More precisely, let the indirect utility distribution function $F^\phi(. \vert c, \pi)$ 
 specify the distribution of utility for individuals of type $c$ at the $\pi$th quantile ($0 \leq \pi \leq 1$) of $F^{c,\phi}_{E}$.
Equalizing opportunities means choosing the policy $\phi$ to equalize utility distributions, $F^\phi(. \vert c, \pi)$, across types at fixed levels of $\pi$:\footnote{Note that in Roemer's original work, utility is assumed to be a deterministic function of $c,e,\phi$. Here we changed the definition slightly to allow for the possibility of non-deterministic dependence.}
\begin{definition}[Luck Egalitarian Equality of Opportunity (e-EOP)]
A policy $\phi$ satisfies Luck Egalitarian EOP if for all $\pi \in [0,1]$ and any two circumstances $c,c'$:
$$F^\phi(. \vert c, \pi) = F^\phi(. \vert c', \pi).$$
\end{definition}
To better understand the subtle difference between Rawlsian EOP and luck egalitarian EOP, consider the following example: suppose in the context of employment decisions, we consider years of education as effort, and gender as circumstance. Suppose Alice and Bob both have 5 years of education, whereas Anna and Ben have 3 and 7 years of education, respectively. Rawlsian EOP would require Alice and Bob to have the same employment prospects, so it would ensure that factors such as sexism wouldn't affect Alice's employment chances, negatively (compared to Bob). Luck egalitarian EOP goes a step further and calculates everyone's rank (in terms of years of education) among all applicants of their gender. In our example, Alice is ranked 1st and Anna is ranked 2nd. Similarly, Bob is ranked 2nd and Ben is ranked 1st. A luck egalitarian EOP policy would ensure that Alice and Ben have the same employment prospects, and may indeed assign Bob to a less desirable position than Alice---even though they have similar years of education. 

Next, we will discuss the above two refinements of substantive EOP in the context of supervised learning.

\section{Setting}\label{sec:model}
As a running example in this section, we consider a business owner who uses A3DM to make salary decisions so as to improve business productivity/revenue. We assume a higher salary is considered to be more desirable by all employees. An A3DM is designed to predict the salary that would improve the employee's performance at the job, using historical data. This target variable, as we will shortly formalize, does not always coincide with the salary the employee is morally accountable/qualified for.

We consider the standard supervised learning setting. A learning algorithm receives a training data set $T=\{(\vx_i,y_i)\}_{i=1}^n$ consisting of $n$ instances, where $\vx_i \in \cX$ specifies the feature vector for individual $i$ and $y_i \in \cY$, the true label for him/her (the salary that would improve his/her performance). Unless otherwise specified, we assume $\cY = \{0,1\}$ and $\cX = \reals^k$. Individuals are assumed to be sampled i.i.d. from a distribution $F$. The goal of a learning algorithm is to use the training data $T$ to fit a \emph{model} (or pick a hypothesis) $h: \cX \rightarrow \cY$ that accurately predicts the label for new instances. Let $\cH$ be the hypothesis class consisting of all the models the learning algorithm can choose from.
A learning algorithm receives $T$ as the input; then utilizes the data to select a model $h \in \cH$ that minimizes some empirical loss, $\cL(T,h)$. We denote the predicted label for an individual with feature vector $\vx$ by $\hy$ (i.e. $\hy = h(\vx)$). 

Consider an individual who is subject to algorithmic decision making in this context. To discuss EOP, we begin by assuming that his/her observable attributes, $\vx$, can be partitioned into two disjoint sets, $\vx= \langle \vz, \vw\rangle$, where $\vz \in \cZ$ denotes the individual's observable characteristics for which he/she is considered morally \emph{not} accountable---this could include sensitive attributes such as race or gender, as well as less obvious attributes, such as zip code. We refer to $\vz$ as morally \emph{arbitrary} or \emph{irrelevant} features. 
%
Let $\vw \in \cW$ denote observable attributes that are deemed morally acceptable to hold the individual accountable for; in the running example, this could include the level of job-related education and experience. We refer to $\vw$ as \emph{accountability} or \emph{relevant} features.
We emphasize once again that determining what factors should belong to each category is entirely outside the scope of this work. We assume throughout that a resolution has been previously reached in this regard---through the appropriate process---and is given to us. 

Let $d \in [0,1]$ specify the individual's \emph{effort-based utility}---the utility he/she should receive solely based on their accountability factors (e.g. the salary an employee should receive based on his/her years of education and job-related experience. Note that this may be different from their actual salary). Effort-based utility $d$ is not directly observable, but we assume it is estimated via a function $g:\cX \times \cY \times \cH \rightarrow \reals^+$, such that $$d = g(\vx, y, h).$$
Function $g$ links the observable information, $\vx,y,$ and $h$, to the effort-based utility, $d$.
%
Let $a \in [0,1]$ be the \emph{actual utility} the individual receives subsequent to receiving prediction $\hy$ (e.g. the utility they get as the result of their predicted salary). We assume there exists a function $f:\cX \times \cY \times \cH \rightarrow \reals^+$ that estimates $a$: 
$$a = f(\vx, y, h).$$
 Throughout, for simplicity we assume higher values of $a$ and $d$ correspond to more desirable conditions.

Let $u$ be the \emph{advantage} or overall \emph{utility} the individual earns as the result of being subject to predictive model $h$. 
For simplicity and unless otherwise specified, we assume $u$ has the following simple form: 
\begin{equation}\label{eq:util}
u = a - d.
\end{equation}
That is, $u$ captures the discrepancy between an individual's \emph{actual} utility ($a$) and their \emph{effort-based} utility $d$. 
With this formulation, an individual's utility is 0 when their actual and effort-based utilities coincide (i.e. $u=0$ if $a = d$). 


We consider the predictive advantage $u$ to be the currency of equality of opportunity for supervised learning. That is, $u$ is what we hope to equalize across similar individuals (similar in terms of what they can be held accountable for).  Our moral argument for this choice is as follows: the predictive model $h$ inevitably makes errors in assigning individuals to their effort-based utilities---this could be due to the target variable not properly reflecting effort-based utility, the prediction being used improperly, or simply a consequence of generalization. Sometimes these errors are beneficial to the subject, and sometimes they cause harm. Advantage $u$ precisely captures this benefit/harm. EOP in this setting requires that all individuals, who do not differ in ways for which they can be held morally accountable, have the same prospect of earning the advantage $u$---regardless of their irrelevant attributes. 
%
%
As an example, let's assume the true labels in the training data reflects individuals' effort-based utilities (as we will shortly argue, this assumption is not always morally acceptable, but for now let's ignore this issue). In this case, a perfect predictor---one that correctly predicts the true label for every individual---will distribute no predictive advantage, but such predictor almost never exists in real world applications. The deployed predictive model almost always distributes some utility among decision-subjects through the errors it makes. A fair model (with EOP rationale) would give all individuals with similar true labels the same prospect of earning this advantage---regardless of their irrelevant attributes. 
%
%

Our main conceptual contribution is to map the above setting to that of economic models of EOP (Section~\ref{sec:related}). We treat the predictive model $h$ as a policy, arbitrary features $\vz$ as circumstance, and the effort-based utilities $d$ as effort (Figure~\ref{fig:map}). In the next Section, we show that through our proposed mapping, most existing statistical notions of fairness can be interpreted as special cases of EOP.


\section{EOP for Supervised Learning}
In this Section, we show that many existing notions of algorithmic fairness, such as statistical parity~\citep{kamiran2009classifying,kamishima2011fairness,feldman2015certifying}, equality of odds~\citep{hardt2016equality}, equality of accuracy~\citep{buolamwini2018gender}, and predictive value parity~\citep{kleinberg2016inherent,zafar2017fairness,zafar2017dmt}, can be cast as special cases of EOP. The summary of our results in this Section can be found in Table~\ref{tab}.
To avoid any confusion with the notation, we define random variables $\vX, Y$ to specify the feature vector and true label for an individual drawn i.i.d. from distribution $F$. Similarly given a predictive model $h$, random variables $\hY=h(\vX), A^h, D^h, U^h$ specify the predicted label, actual utility, the effort-based utility, and advantage, respectively, for an individual drawn i.i.d. from $F$. When the predictive model in reference is clear from the context, we drop the superscript $h$ for brevity.
\begin{table*}
	\centering
    \begin{tabular}{ | l  |  c |  c | c |}
    \hline
    Notion of fairness &  Effort-based utility $D$ & Actual utility $A$ & Notion of EOP \\ \hline
    Accuracy Parity &  constant (e.g. $0$) & $(\hY - Y)^2$ & Rawlsian\\
    Statistical Parity & constant (e.g. $1$) & $\hY$ & Rawlsian\\
    Equality of Odds & $Y$ & $\hY$ & Rawlsian\\
    Predictive Value Parity  & $\hY$ & $Y$ & egalitarian\\
    \hline
    \end{tabular}
     \caption{Interpretation of existing notions of algorithmic fairness for binary classification as special instances of EOP.}
    \label{tab}
\end{table*}

\hl{
Before we formally establish a connection between algorithmic fairness and EOP, we shall briefly overview the Fair ML literature and remind the reader of the precise definition of previously-proposed notions of fairness. 
Existing notions of algorithmic fairness can be divided into two distinct categories: \emph{individual}-~\citep{dwork2012fairness, speicher2018a} and \emph{group}-level fairness.
Much of the existing work on algorithmic fairness has been devoted to the study of group (un)fairness, also called \emph{statistical} unfairness or \emph{discrimination}. Statistical notions of fairness require that given a classifier, a certain fairness metric is equal across all (protected or socially salient) groups.
More precisely, assuming $\vz \in \cZ$ specifies the group each individual belongs to, statistical parity seeks to equalize the percentage of people receiving a particular outcome across different groups: 
\begin{definition}[Statistical Parity]
A predictive model $h$ satisfies statistical parity if $\forall  \vz, \vz' \in \cZ, \forall \hy \in \cY:$
$$\PR_{(\vX, Y) \sim F}[h(\vX) = \hy \vert \vZ = \vz]  = \PR_{(\vX, Y) \sim F}[h(\vX) = \hy \vert \vZ = \vz'].$$
\end{definition}
Equality of odds requires the equality of false positive and false negative rates across different groups:
\begin{definition}[Equality of Odds]
A predictive model $h$ satisfies equality of odds if $\forall  \vz, \vz' \in \cZ, \forall y, \hy \in \cY:$
$$\PR_{(\vX, Y) \sim F}[\hY = \hy \vert \vZ = \vz, Y=y]  = \PR_{(\vX, Y) \sim F}[\hY = \hy \vert \vZ = \vz', Y = y].$$
\end{definition}
Equality of accuracy requires the classifier to make equally accurate predictions across different groups:
\begin{definition}[Equality of Accuracy]
A predictive model $h$ satisfies equality of accuracy if $\forall  \vz, \vz'  \in \cZ:$
$$\Exp_{(\vX, Y) \sim F}[ (\hY - Y)^2 \vert \vZ = \vz] = \Exp_{(\vX, Y) \sim F}[ (\hY - Y)^2 \vert \vZ = \vz'].$$
\end{definition}
Predictive value parity (which can be thought of as a weaker version of calibration~\citep{kleinberg2016inherent}) requires the equality of positive and negative predictive values across different group:
\begin{definition}[Predictive Value Parity]
A predictive model $h$ satisfies predictive value parity if $\forall  \vz, \vz' \in \cZ, \forall y, \hy \in \cY:$
$$\PR_{(\vX, Y) \sim F}[Y = y \vert \vZ = \vz, \hY=\hy]  = \PR_{(\vX, Y) \sim F}[Y = y \vert \vZ = \vz', \hY=\hy].$$
\end{definition}
}
\subsection{Statistical Parity, Equality of Odds and Accuracy as Rawlsian EOP}\label{sec:rawlsian}
We begin by translating Rawlsian EOP into the supervised learning setting using the mapping proposed in Figure~\ref{fig:map}. Recall that we proposed \emph{replacing $e$ with effort-based utility $d$}, and circumstance $c$ with vector of irrelevant features $\vz$.
In order for the definition of Rawlsian EOP to be morally acceptable, we need $d$ to not be affected by $\vz$ and the model $h$. In other words, it can only be a function of $\vw$ and $y$. Let $F^h(.)$ specify the distribution of utility across individuals under predictive model $h$. We define Rawlsian EOP for supervised learning as follows:
\begin{definition}[R-EOP for supervised learning]\label{eq:rawlsian}
Suppose $d = g(\vw,y)$. 
Predictive model $h$ satisfies Rawlsian EOP if for all $\vz,\vz' \in \cZ$ and all $d \in [0,1]$, $$F^h(. \vert \vZ = \vz, D = d) = F^h(. \vert \vZ = \vz', D = d).$$
\end{definition}
In the binary classification setting, if we assume the true label $Y$ reflects an individual's effort-based utility $D$, Rawlsian EOP translates into equality of odds across protected groups:\footnote{Note that \citet{hardt2016equality} referred to a weaker measure of algorithmic fairness (i.e. equality of true positive rates) as equality of opportunity.}
\begin{proposition}[Equality of Odds as R-EOP]
Consider the binary classification task where $\cY = \{0,1\}$.
Suppose $U = A - D$, $A = h(\vX) = \hY$ (i.e., the actual utility is equal to the predicted label) and $D =g(\vW,Y)$ where $g(\vW,Y) = Y$ (i.e., effort-based utility of an individual is assumed to be the same as their true label). Then the conditions of R-EOP are equivalent to those of equality of odds.
\end{proposition}
\begin{proof}
Recall that R-EOP requires that $\forall  \vz, \vz' \in \cZ, \forall d \in \cD$, and for all possible utility levels $u$:
$$\PR( U \leq u \vert \vZ = \vz, D = d) = \PR( U \leq u \vert \vZ = \vz', D = d).$$ 
Replacing $U$ with $(A - D)$, $D$ with $Y$, $A$ with $\hY$, the above is equivalent to 
\begin{eqnarray*}
&& \forall  \vz, \vz' \in \cZ, \forall y \in \{0,1\}, \forall u \in \{0,\pm 1\}:  \PR[ \hY - Y \leq u \vert \vZ = \vz, Y=y] = \PR[ \hY - Y \leq u \vert \vZ = \vz', Y=y]\\
&\Leftrightarrow & \forall  \vz, \vz' \in \cZ, \forall y \in \{0,1\}, \forall u \in \{0,\pm 1\}:  \PR[\hY \leq u + y \vert \vZ = \vz, Y=y] = \PR[\hY \leq u + y \vert \vZ = \vz', Y = y]\\
&\Leftrightarrow & \forall  \vz, \vz' \in \cZ, \forall y \in \{0,1\}, \forall \hy \in \{0,1\}:  \PR[\hY = \hy \vert \vZ = \vz, Y=y]  = \PR[\hY = \hy \vert \vZ = \vz', Y = y]
\end{eqnarray*}
where the last line is identical to the conditions of equality of odds for binary classification.
\end{proof}
The important role of the above proposition is to explicitly spell out the moral assumption underlying equality of odds as a measure of fairness: by measuring fairness through equality of odds, we implicitly assert that all individuals with the same true label have the same effort-based utility. This can clearly be problematic in practice: true labels don't always reflect/summarize accountability factors. At best, they are only a reflection of the current state of affairs---which itself might be tainted by past injustices. For these reasons, we argue that equality of odds can only be used as a valid measure of algorithmic fairness (with an EOP rationale) once the validity of the above moral equivalency assumption has been carefully investigated and its implications are well understood in the specific context it is utilized in.

Other statistical definitions of algorithmic fairness---namely statistical parity and equality of accuracy---can similarly be thought of as special instances of R-EOP. See Table~\ref{tab}. 
For example statistical parity can be interpreted as R-EOP if we assume all individuals have the same effort-based utility.\footnote{Statistical parity can be understood as equality of outcomes as well, if we assume $\hY$ reflects the outcome.} 
\begin{proposition}[Statistical Parity as R-EOP]
Consider the binary classification task where $\cY = \{0,1\}$.
Suppose $U = A - D$, $A = \hY$ and $D =g(\vW,Y)$ where $g(\vW,Y)$ is a constant function (i.e., effort-based utility of all individuals is assumed to be the same). Then the conditions of R-EOP is equivalent to statistical parity.
\end{proposition}
\begin{proof}
Without loss of generality, suppose $g(\vX,Y,h) \equiv 1$, i.e. all individuals effort-based utility 1.
Recall that R-EOP requires that $\forall  \vz, \vz' \in \cZ, \forall d \in \cD, \forall u \in \reals:$
$$ \PR( U \leq u \vert \vZ = \vz, D = d) = \PR( U \leq u \vert \vZ = \vz', D = d).$$ 
Replacing $U$ with $(A - D)$, $D$ with $1$, and $A$ with $\hY$, the above is equivalent to 
\begin{eqnarray*}
&& \forall  \vz, \vz' \in \cZ, \forall d \in \{1\}, \forall u \in \{0,-1\}: \PR[ \hY - D \leq u \vert \vZ = \vz, D=1] = \PR[ \hY - D \leq u \vert \vZ = \vz', D=1]\\
&\Leftrightarrow & \forall  \vz, \vz' \in \cZ, \forall u \in \{0, -1\}: \PR[\hY \leq u + 1 \vert \vZ = \vz] = \PR[\hY \leq u + 1 \vert \vZ = \vz']\\
&\Leftrightarrow & \forall  \vz, \vz' \in \cZ, \forall \hy \in \{0, 1\}:  \PR[\hY = \hy \vert \vZ = \vz]  = \PR[\hY = \hy \vert \vZ = \vz']
\end{eqnarray*}
where the last line is identical to the conditions of statistical parity for binary classification.
\end{proof}
As a real-world example where statistical parity can be applied, consider the following: suppose the society considers all patients to have the same effort-based utility---which can be enjoyed by access to proper clinical examinations. Now suppose that undergoing an invasive clinical examination has utility 1 if one has the suspected diseases and -1 otherwise, whereas avoiding the same clinical investigation has utility 1 if one does not have the suspected disease, and -1 otherwise. For all subjects, the effort-based utility is the same (the maximum utility, let us suppose). In other words, all people with a disease deserve the invasive clinical investigation and all people without the disease deserve to avoid it.  Consider a policy of giving clinical investigation to all the people without the disease and to no people without the disease. This would achieve an equal distribution of effort-based utility ($D$) and distribute no advantage $U$. Such policy, however, could only be achieved with a perfect accuracy predictor.  For an imperfect accuracy predictor, R-EOP would require the distribution of (negative, in this case) utility (U) to give the same chance to African Americans and white patients with (without) the disease to receive (avoid) an invasive clinical exam.

\begin{proposition}[Equality of Accuracy as R-EOP]
Consider the binary classification task where $\cY = \{0,1\}$.
Suppose $U = A - D$, $A = (\hY - Y)^2$ and $D = g(\vW,Y)$ where $g(\vW,Y) \equiv 0$  (i.e., effort-based utility of all individuals are assumed to be the same and equal to $0$). Then the conditions of R-EOP is equivalent to equality of accuracy.
\end{proposition}
\begin{proof}
Recall that R-EOP requires that $\forall  \vz, \vz' \in \cZ, \forall d \in \cD, \forall u \in \reals:$
$$ \PR( U \leq u \vert \vZ = \vz, D = d) = \PR( U \leq u \vert \vZ = \vz', D = d).$$ 
Replacing $U$ with $(A - D)$, $D$ with $0$, and $A$ with $(\hY - Y)^2$, the above is equivalent to $\forall  \vz, \vz' \in \cZ, \forall d \in \{0\}, \forall u \in \{0, 1\}:$
$$\PR[ (\hY - Y)^2 - D \leq u \vert \vZ = \vz, D = d] = \PR[ (\hY - Y)^2 - D \leq u \vert \vZ = \vz', D = d]$$
We can then write:
\begin{eqnarray*}
&\Leftrightarrow & \forall  \vz, \vz', \forall u: \PR[ (\hY - Y)^2 = u \vert \vZ = \vz] = \PR[ (\hY - Y)^2 = u \vert \vZ = \vz']\\
&\Leftrightarrow & \forall  \vz, \vz' \in \cZ: \Exp[ (\hY - Y)^2 \vert \vZ = \vz] = \Exp[ (\hY - Y)^2 \vert \vZ = \vz']
\end{eqnarray*}
where the last line is identical to the conditions of equality of accuracy for binary classification.
\end{proof}
The critical moral assumption underlying equality of accuracy as a measure of fairness (with EOP rationale) is that errors reflect the advantage distributed by the predictive model among decision subjects. This exposes the fundamental ethical problem with adopting equality of accuracy as a measure of algorithmic fairness: it fails to distinguish between errors that are beneficial to the subject and those that are harmful. For example, in the salary prediction example, equality of accuracy would make no distinction between an individual who earns a salary higher than what they deserve, and someone who earns lower than their effort-based/deserved salary. 

\hl{
\paragraph{Max-min distribution vs. strict equality}
At a high level, R-EOP prescribes \emph{equalizing} advantage distribution across persons with the same effort-based utility. Some egalitarian philosophers have argued that we can remain faithful to the spirit (though not the letter) of EOP by delivering a \emph{max-min} distribution of advantage, instead of a strict egalitarian one~\citep{roemer2002equality}. The max-min distribution deviates from equality only when this makes the worst off group better off. Even though this distribution permits inequalities that do not reflect accountability factors, it is considered a morally superior alternative to equality, if it improves the utility of least fortunate.\footnote{The idea that inequalities are justifiable only when they result from a scheme arranged to maximally benefit the worst off position is expressed through the \emph{Difference Principle} by John Rawls in his theory of ``justice as fairness"~\citep{rawls1958justice}.} The max-min distribution addresses the ``leveling down" objection to equality: the disadvantaged group may be more interested in maximizing their absolute level of utility, as opposed to their relative utility compared to that of the advantaged group.
}

\subsection{Predictive Value Parity as Egalitarian EOP}\label{sec:calibration}

Note that predictive value parity (equality of positive and negative predictive values across different groups) can not be thought of as an instance of R-EOP, as it requires the effort-based utility of an individual to be a function of the predictive model $h$ (as we will shortly show, it assumes $D = h(\vX)$). This is in violation of the absolutist view of Rawlsian EOP. In this Section, we show that predictive value parity can be cast as an instance of luck egalitarian EOP.

We first specialize \citeauthor{roemer2002equality}'s model of Egalitarian EOP to the supervised learning setting.
Recall that egalitarian EOP allows the effort-based utility to be a function of the predictive model $h$, that is $D = f(\vX, Y,h)$. 
When this is the case, following the argument put forward by Roemer we posit that the distribution of effort-based utility for a given type $\vz$ (denoted by  $F^{\vz,h}_{D}$) is a characteristic of the type $\vz$, not something for which any individual belonging to the type can be held accountable. Therefore, an inter-type comparable measure of effort-based utility must factor out the goodness or badness of this distribution. We consider two individuals as being equally deserving if they sit at the same quantile or rank of the distribution of $D$ for their corresponding type. 

More formally, let the \emph{indirect utility distribution} function, denoted by $F^h(. \vert \vz, \pi)$, specify the distribution of utility for individuals of type $\vz$ at the $\pi$th quantile ($0 \leq \pi \leq 1$) of effort-based utility distribution, $F^{\vz,h}_{D}$.
Equalizing opportunities means choosing the predictive model $h$ to equalize the indirect utility distribution across types, at fixed levels of $\pi$:
\begin{definition}[e-EOP for supervised learning]
Suppose $d = f(\vx,y,h)$. 
Predictive model $h$ satisfies egalitarian EOP if for all $\pi \in [0,1]$ and $\vz, \vz' \in \cZ$,
\begin{equation}\label{eq:egalitarian}
F^h(. \vert \vZ=\vz, \Pi = \pi) = F^h(. \vert \vZ=\vz', \Pi = \pi).
\end{equation}
\end{definition}
Next, we show that predictive value parity can be thought of as a special case of e-EOP, where the predicted label/risk $h(\vX)$ is assumed to reflect the individual's effort-based utility, and the true label $Y$ reflects his/her actual utility.
\begin{proposition}[predictive value parity as e-EOP]\label{prop:calibration}
Consider the binary classification task where $\cY = \{0,1\}$.
Suppose $U = A - D$, $A = Y$ and $D = g(\vX,Y,h)$ where $g(\vX,Y,h) = h(\vX) = \hY$ (i.e., effort-based utility of an individual under $h$ is assumed to be the same as their predicted label). Then the conditions of e-EOP are equivalent to those of predictive value parity.
\end{proposition}
\begin{proof}
Recall that e-EOP requires that $\forall  \vz, \vz' \in \cZ, \forall \pi \in [0,1],$ and $\forall u \in \reals: $
$$\PR[ U \leq u \vert \vZ = \vz, \Pi=\pi] = \PR[ U \leq u \vert \vZ = \vz', \Pi=\pi].$$ 
Note that since $D = \hY$ and in the binary classification, $\hY$ can only take on two values, there are only two ranks/quantiles possible in terms of the effort-based utility---corresponding to $\hY = 0$ and $\hY = 1$. So the above condition is equivalent to $\forall  \vz, \vz' \in \cZ, \forall \hy \in \{0,1\}, \forall u \in \{0,\pm 1\}: $
$$\PR[ U \leq u \vert \vZ = \vz, \hY=\hy] = \PR[ U \leq u \vert \vZ = \vz', \hY=\hy].$$ 
Replacing $U$ with $(A - D)$, $D$ with $\hY$, $A$ with $Y$, the above is equivalent to 
\begin{eqnarray*}
&& \forall  \vz, \vz' \in \cZ, \forall \hy \in \{0,1\}, \forall u \in \{0,\pm 1\}: \PR[  Y - \hY \leq u \vert \vZ = \vz, \hY=\hy] = \PR[ Y - \hY\leq u \vert \vZ = \vz', \hY=\hy]\\
&\Leftrightarrow & \forall  \vz, \vz' \in \cZ, \forall \hy \in \{0,1\}, \forall u \in \{0,\pm 1\}: \PR[Y \leq u + \hy \vert \vZ = \vz, \hY=\hy] = \PR[Y \leq u + \hy \vert \vZ = \vz', \hY=\hy]\\
&\Leftrightarrow &  \forall  \vz, \vz' \in \cZ, \forall \hy \in \{0,1\}, \forall y \in \{0, 1\}: \PR[Y = y \vert \vZ = \vz, \hY=\hy]  = \PR[Y = y \vert \vZ = \vz', \hY=\hy]
\end{eqnarray*}
where the last line is identical to predictive value parity.
\end{proof}

Note that there are two assumptions needed to cast predictive value parity as an instance of e-EOP: 1) the predicted label/risk $h(\vx)$ reflects the individual's effort-based utility; and 2) the true label $Y$ reflects his/her actual utility. The plausibility of such moral assumptions must be critically evaluated in a given context before predictive parity can e employed as a valid measure of fairness. Next, we discuss the plausibility of these assumptions through several real-world examples.

\hl{
\paragraph{Plausibility of assumption 1} The choice of predicted label, $h(\vX)$, as the indicator of effort-based utility, may sound odd at first. However, there are several real-world settings in which this assumption is considered appropriate.
Consider the case of driving under influence (DUI): the law considers all drivers equally at risk of causing an accident---due to the consumption of alcohol or drugs---equally accountable for their risk and punishes them similarly, even though only some of them will end up in an actual accident, and the rest won't. In this context, the potential/risk of causing an accident---as opposed to the actual outcome---justifies unequal treatment, because we believe differences in actual outcomes among equally risky individuals is mainly driven by arbitrary factors, such as brute luck. Arguably, such factors should never specify accountability. 
}

\hl{
\paragraph{Can assumptions 1 and 2 hold simultaneously?}
The following is an example in which ssumptions 1 and 2 hold simultaneously (in particular, the true label $Y$ specifies the actual utility an individual receives subsequent to being subject to automated decision making).
Consider the students of a course, offered online and open to students from all over the world. The final assessment of students enrolled in the course includes an essential oral exam. The oral exam is very challenging and extremely competitive. The instructors hold an exam session every month. Every student is allowed to take the oral exam, but since resources for oral examinations are limited, to discourage participation without preparation, the rule is that, if a student fails the exam, he/she has to wait one year before taking the exam again.
Suppose that students belong to one of the two groups: African Americans and Asians. African American and Asian students study in different ways, with different cognitive strategies. As a result, an African American student with 0.9 passing score may correspond to a very different feature vector compared to an Asian student with a 0.9 passing score. 
Suppose a predictive model is used to predict the outcome of the oral exam for individual students, based on the student's behavioral data. (The online learning platform records data on how students interact with the course materials.) Students are given a simple ``pass/fail" prediction to help them make an informed choice about when to take the exam. 
In this example, we argue that both assumptions underlying predictive value parity are satisfied:
\begin{enumerate}
\item $A = Y$: Passing the exam is a net utility, not passing the exam is a net disutility (due to the one year delay). Also, being predicted to pass per se has no utility associated with it.
\item $D = \hY$. It is plausible to consider students morally responsible for their chances of success, because the predictions are calculated based on how they have studied the course material.
\end{enumerate}
In this example, a fair predictor (with EOP rationale) should satisfies predictive value parity. That means: students who are predicted to pass, should be equally likely to pass the exam, irrespective of their race. 
}


\paragraph{On Recent Fairness Impossibility Results} 
Several papers have recently shown that group-level notions of fairness, such as predictive value parity and equality of odds, are generally incompatible with one another and cannot hold simultaneously~\citep{kleinberg2016inherent,friedler2016possibility}.
Our approach confers a moral meaning to these impossibility results: they can be interpreted as contradictions between fairness desiderata reflecting different and irreconcilable moral assumptions. For example predictive value parity and equality of odds make very different assumptions about the effort-based utility $d$: Equality of odds assumes all persons with similar true labels are equally accountable for their labels, whereas predictive value parity assumes all persons with the same predicted label/risk are equally accountable for their predictions. 
Note that depending on the context, usually only one (if any) of these assumptions is morally acceptable. We argue, therefore, that unless we are in the highly special case where $Y = h(\vX)$, it is often unnecessary---from a moral standpoint---to ask for both of these fairness criteria to be satisfied simultaneously.

\begin{figure*}[t!]
    \centering
    \begin{subfigure}[b]{0.32\textwidth}
        \includegraphics[width=\textwidth]{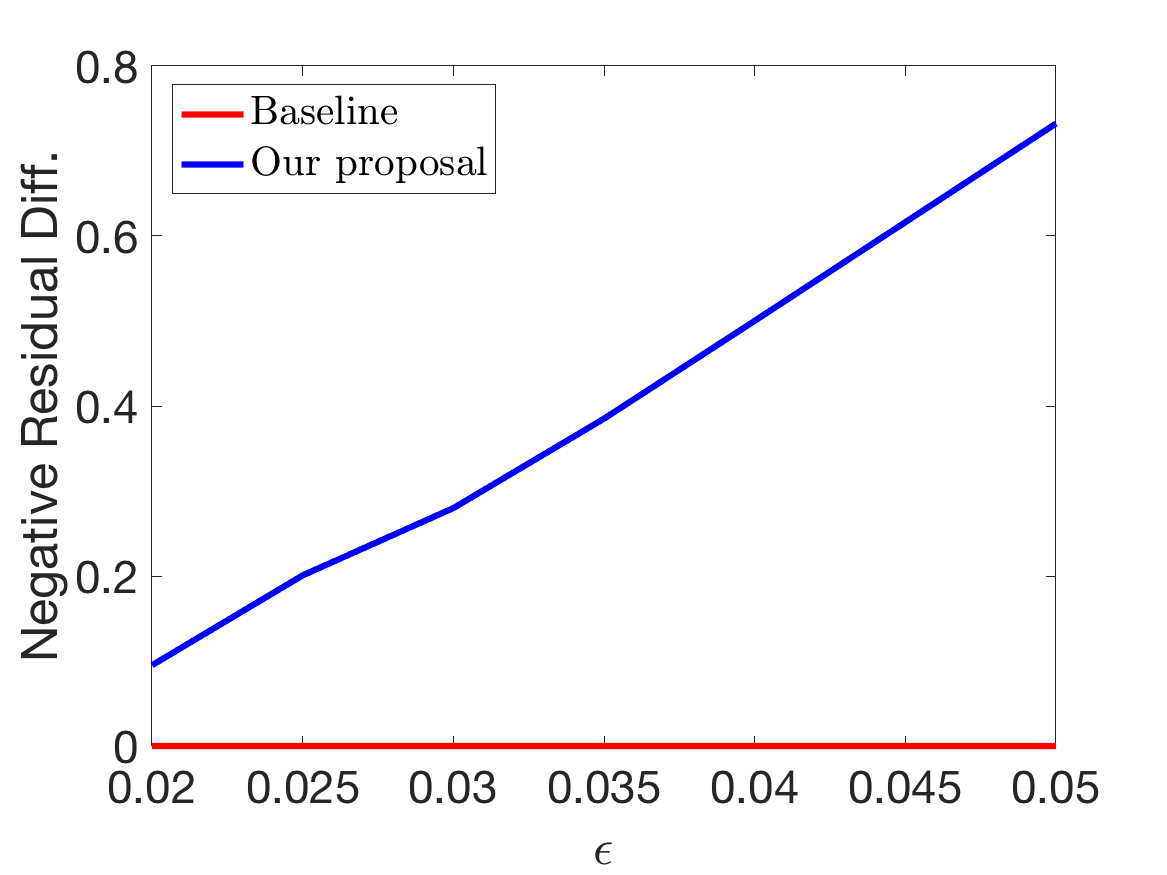}
        \caption{ }
        \label{fig:crime_nrd}
    \end{subfigure}
    \begin{subfigure}[b]{0.32\textwidth}
        \includegraphics[width=\textwidth]{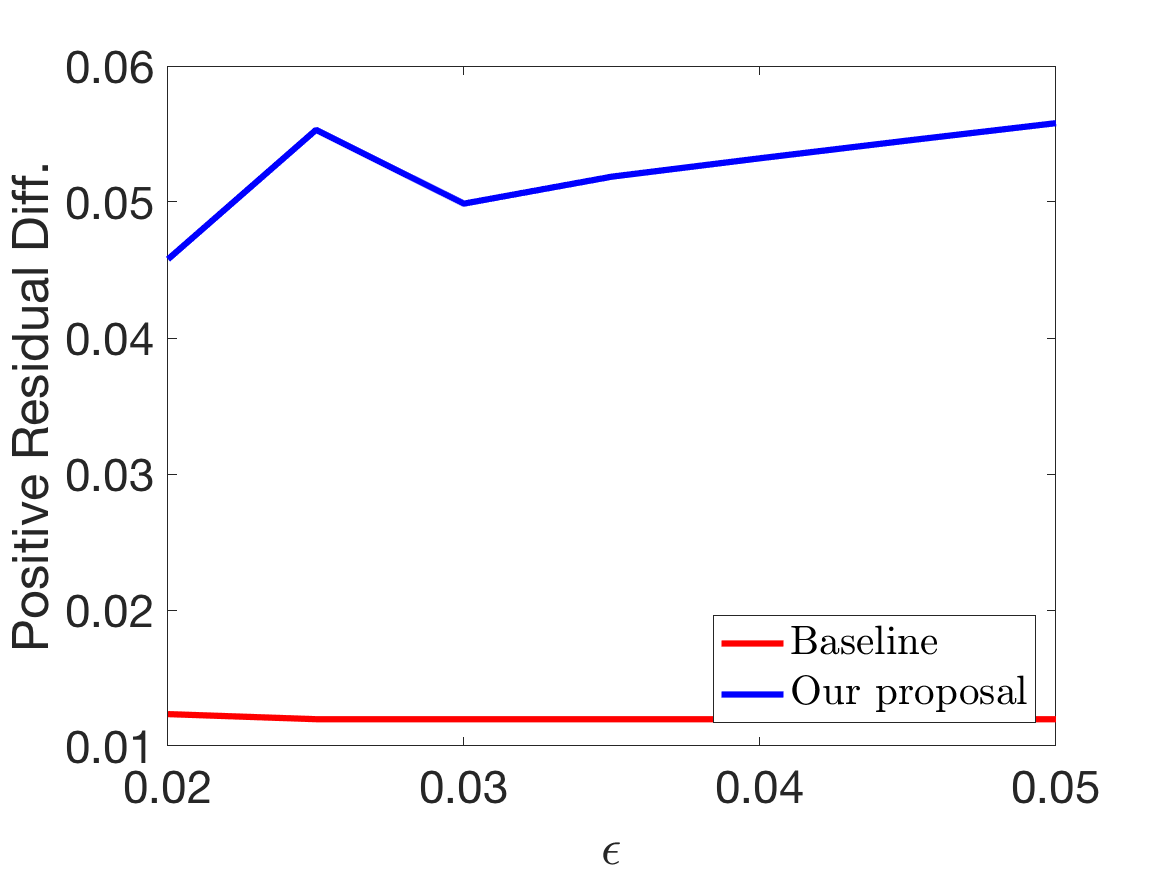}
        \caption{ }
        \label{fig:crime_prd}
    \end{subfigure}
    \begin{subfigure}[b]{0.32\textwidth}
        \includegraphics[width=\textwidth]{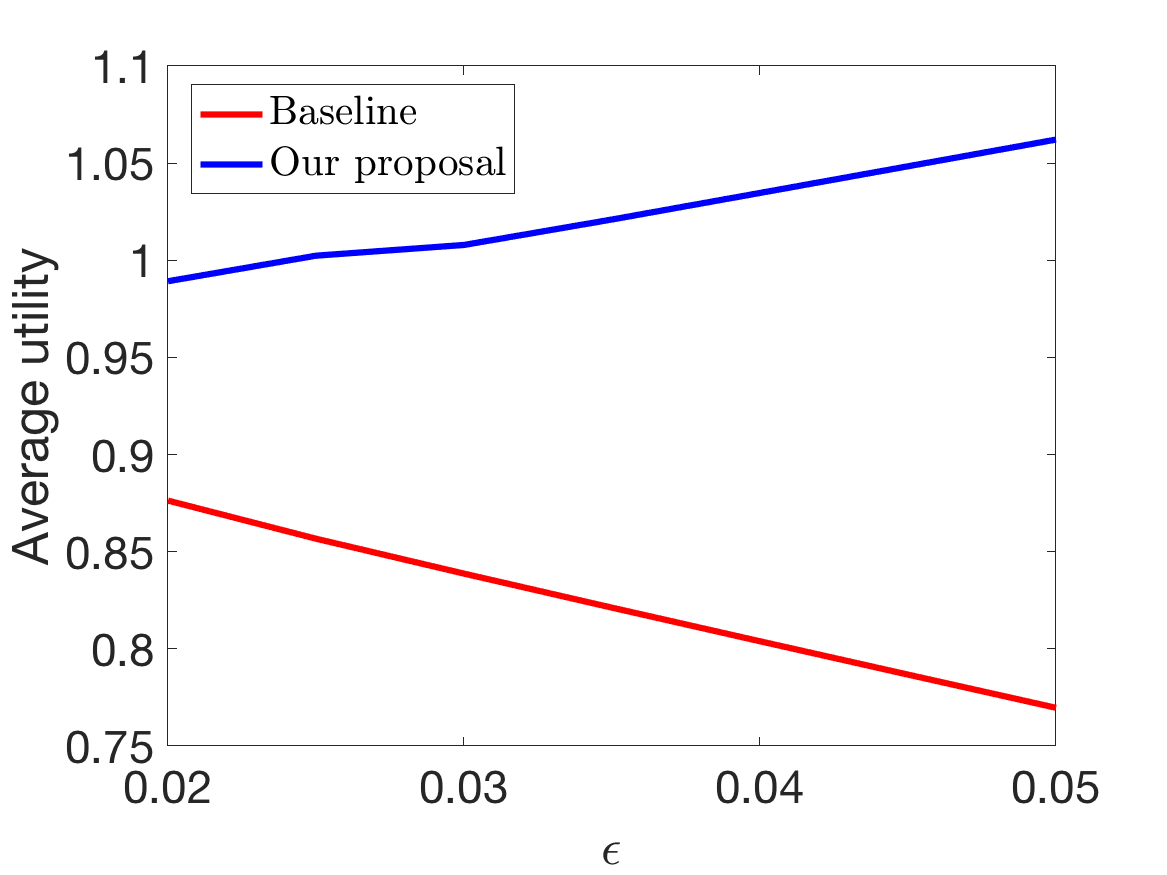}
        \caption{ }
        \label{fig:crime_util}
    \end{subfigure}
       \caption{NRD, PRD, and average utility of the disadvantaged group as a function of $\epsilon$ (the upperbound on mean squared error). The notion of fairness enforced on algorithmic decisions can have a devastating impact on the welfare of the disadvantaged group.}
       \label{fig:experiments}
\end{figure*}

\section{Egalitarian Measures of Fairness}\label{sec:egalitarian}
In this section, inspired by Roemer's model of egalitarian EOP we present a new family of measures for algorithmic fairness. Our proposal is applicable to supervised learning tasks beyond binary classification, and to utility functions beyond the simple linear form specified in Equation~\ref{eq:util}. We illustrate our proposal empirically, and compare it with existing notions of (un)fairness for regression. Our empirical findings suggest that employing a measure of algorithmic (un)fairness when its underlying assumptions are not met, can have devastating consequences on the welfare of decision subjects.

\subsection{A New Family of Measures}
For supervised learning tasks beyond binary classification (e.g. multiclass classification or regression), the requirement of equation \ref{eq:egalitarian} becomes too stringent, as there will be (infinitely) many quantiles to equalize utilities over. The problem persists even if we relax the requirement of equal utility distributions to maximizing the minimum expected utility at each quantile. More formally,
let $v^\vz(\pi,h)$ specify the expected utility of individuals of type $\vz$ at the $\pi$th quantile of the effort-based utility distribution. 
For $\pi \in [0,1]$, we say that a predictive model $h^\pi$ satisfies egalitarian EOP at the $\pi$-slice of the population, if:
$$h^\pi \in \arg\max_{h \in \cH} \min_{\vz \in \cZ} v^\vz(\pi,h).$$
Assuming we are concerned only with the $\pi$-slice, then $h^\pi$ would be the equal-opportunity predictive model. Unfortunately, when we move beyond binary classification, we generally cannot find a model that is simultaneously optimal for all ranks $\pi \in [0,1]$. Therefore, we need to find a compromise. Following \citeauthor{roemer2002equality}, we  define the e-EOP predictive model as follows: 
\begin{equation}\label{eq:int_1}
h^* \in \arg\max_{h \in \cH} \min_{\vz \in \cZ} \int_0^1 v^\vz(\pi,h) d\pi.
\end{equation}
That is, we consider $h^*$ to be an e-EOP predictive model if it maximizes the expected utility of the \emph{worst off group} (i.e. $\int_0^1 v^\vz(\pi,h) d\pi$).\footnote{Roemer in fact proposes two further alternatives: in the first solution, the objective function for each $\pi$-slice of the population is assumed to be $\min_{\vz \in \cZ} v^\vz(\pi,h)$---which is then weighted by the size of the slice. In the second solution, he declares the equal opportunity policy to be the average of the policies $h^\pi$. Roemer expresses no strong preference for any of these alternatives, other than the fact that computational simplicity sometimes suggests one over the others~\citep{roemer2002equality}. This is in fact the reasoning behind our choice of Equation~\ref{eq:int_1}.}
Replacing the expectation with its in-sample analogue, our proposed family of e-EOP measures can be evaluated on the data set $T$ as follows:
\begin{equation*}
\cF(h,T) = \min_{\vz \in \cZ}  \frac{1}{n_\vz} \sum_{i \in T: \vz_i =\vz} u(\vx_i,y_i,h)
\end{equation*}
where $u(\vx_i,y_i,h)$ is the utility an individual with feature vector $\vx_i$ and true label $y_i$ receives when predictive model $h$ is deployed; and $n_\vz$ is the number of individuals in $T$ whose arbitrary features value is $\vz \in \cZ$. (The arbitrary features value $\vz$ specifies the (intersectional) \emph{group} each individual belongs to. We use $m$ to denote the number of such (intersectional) groups. For simplicity in our illustration, we refer to these groups as $G_1,\cdots,G_m$.)

To guarantee fairness, we propose the following in-processing method: maximize the expected utility of the worst off group, subject to error being upper bounded (by $\epsilon$).
\begin{eqnarray}\label{eq:proposal}
\max_{h \in \cH} && \cF(h,T) \nonumber\\
\text{s.t. }&&  \cL(T,h) \leq \epsilon 
\end{eqnarray}
Note that if the loss function $\cL$ is convex and $\cF$ is concave in model parameters, Optimization~\ref{eq:proposal} is convex and can be solved efficiently.

We remark that our notion of fairness does not require us to explicitly specify the effort-based utility $D$, since it only compares the \emph{overall} expected utility of different groups with one another---without the need to explicitly compare the utility obtained by individuals at a particular rank of $D$ across different groups. Furthermore, the utility function, $u(\vx,y,h)$, does not have to be restricted to take the simple linear form specified in Equation~\ref{eq:util}. 

\subsection{Illustration}\label{sec:experiments}
Next, we illustrate our proposal on the \emph{Crime and Communities data set}~\citep{communities_crime_dataset}. The data consists of 1994 observations, each corresponding to a community/neighborhood in the United States. Each community is described by 101 features, specifying its socio-economic, law enforcement, and crime statistics extracted from the 1995 FBI UCR. Community type (e.g. urban vs. rural), average family income, and the per capita number of police officers in the community are a few examples of the explanatory variables included in the dataset. The target variable ($Y$) is the ``per capita number of violent crimes". We train a linear regression model, $\vtheta \in \reals^k$, on this dataset to predict the per capita number of violent crimes for a new community. We hypothesize that crime predictions can affect the law enforcement resources assigned to the community, the value of properties located in the neighborhood, and business investments drawn to it. 

We preprocess the original dataset as follows: we remove the instances for which target value is unknown. Also, we remove features whose values are missing for more than $80\%$ of instances. We standardize the data so that each feature has mean 0 and variance 1. We divide all target values by a constant so that labels range from $0$ to $1$. 
Furthermore, we flip all labels ($ y \rightarrow 1-y$), so that higher $y$ values correspond to more desirable outcomes.
We assume a neighborhood belongs to the protected group ($G_1$) if the majority of its residents are non-Caucasian, that is, the percentage of African American, Hispanic, and Asian residents of the neighborhood combined, is above $50\%$. This divides the training instances into two groups $G_0, G_1$. We include this group membership information as the (sensitive) feature $z$ in the training data ($z_i = 1[i \in G_1]$).  

For simplicity, we assume the utility function $u$ has the following functional dependence on $\vx$ and $\vtheta$: $u(z,y,\hy)$; that is, $u$'s dependence on $\vx$ and $\vtheta$ are through $z$ and $\hy=\vtheta.\vx$, respectively. 
For communities belonging to $G_0$ and $G_1$, we assume $u(z,y,\hy) = f(z,y,\hy) - g(z,y,\hy)$ is respectively defined as follows:
\begin{itemize}
\item For a majority-Caucasian neighborhood, 
$$u(0, y, \hy) = (1 + 0.5\hy y) - (0.5\hy).$$
\item For a minority-Caucasian neighborhood,
$$u(1, y, \hy) = (1 + 3\hy y + 2\hy) - ( y ).$$ 
\end{itemize}
At a high level, neighborhoods in both groups enjoy a high utility if their predicted and actual crime rates are low, simultaneously (note that the absolute value of utility derived from this case is higher for the minority). The minority-Caucasian group further benefits from low crime predictions (regardless of actual crime rates). We assume the effort-based utility for the minority group, is one minus the actual crime rate ($y$), and for the majority group, it is proportional to one minus the predicted crime rate ($0.5\hy$).
Note that these utility functions are made up for illustration purposes only, and do not reflect any deep knowledge of how crime and law enforcement affect the well-being of a neighborhood's residents. 

To illustrate our proposal, we solve the following \emph{convex} optimization problem for different values of $\epsilon$:
\begin{eqnarray}\label{eq:opt_true}
\max_{\sigma, \vtheta} &&  \sigma \nonumber\\
\text{s.t. }&&  \frac{1}{n_0}\sum_{i \in G_0} -0.5\vtheta. \vx_i  + 0.5(\vtheta. \vx_i) y_i + 1 \geq \sigma   \nonumber\\
&&  \frac{1}{n_1}\sum_{i \in G_1} 2\vtheta. \vx_i + 3(\vtheta. \vx_i) y_i - y_i +1 \geq \sigma   \nonumber\\
&& \frac{1}{n} \sum_{i=1}^n (\vtheta. \vx_i - y_i)^2 + \lambda \Vert \vtheta \Vert_1 \leq \epsilon 
\end{eqnarray}
We choose the value of $\lambda$ by running a 10-fold cross validation on the data set.
For each value of $\epsilon$ (Mean Squared Error), we measure the following quantities via 5-fold cross validation:
\begin{itemize}
%
\item \textbf{Positive residual difference}~\citep{calders2013controlling} is the equivalent of false positive rate in regression, and is computed by taking the absolute difference of mean positive residuals across the two groups:
$$\left|\frac{1}{n_1^+} \sum_{i \in G_1} \max\{0,(\hat{y}_i- y_i)\}  - \frac{1}{n_0^+}\sum_{i \in G_0} \max\{0,(\hat{y}_i -y_i)\}\right|.$$
In the above, $n_g^+$ is the number of individuals in group $g \in \{0,1\}$ who get a positive residual, i.e. $\hat{y}_i - y_i \geq 0$. 
\item \textbf{Negative residual difference}~\citep{calders2013controlling} is the equivalent of false negative rate in regression, and is computed by taking the absolute difference of mean negative residuals across the two groups.
%
\item \textbf{Average utility} of the disadvantaged group is computed by taking the average utility of all individuals in the test data set:
$$ \min\left\lbrace \frac{1}{n_0} \sum_{i \in G_0} u(\vx_i, y_i, h) , \frac{1}{n_1} \sum_{i \in G_1} u(\vx_i, y_i, h) \right\rbrace.$$
\end{itemize}
Figure~\ref{fig:experiments} shows the results of our simulations. Blue curves correspond to our proposal (Optimization~\ref{eq:opt_true}). As evident in Figures~\ref{fig:crime_nrd} and \ref{fig:crime_prd}, positive and negative residual difference increase with $\epsilon$, while the average utility increases (see Figure~\ref{fig:crime_util}).

To compare our proposal with existing measures of (un)fairness for regression, we utilize the in-processing method of \citet{heidari2018fairness}. The method enforces an upperbound on $\sum_i (\hy_i- y_i)$, and has been shown to control the positive and negative residual difference across the two groups. More precisely, we solve the following optimization problem for different values of $\epsilon$:
\begin{equation}\label{eq:opt_wrong}
\max_{\vtheta} \frac{1}{n}\sum_{i \in T}  \vtheta. \vx_i  - y_i \text{  s.t.  } \frac{1}{n} \sum_{i=1}^n (\vtheta. \vx_i - y_i)^2 + \lambda \Vert \vtheta \Vert_1 \leq \epsilon
\end{equation}
Red curves in Figure~\ref{fig:experiments} correspond to this baseline.
%
As evident in Figures~\ref{fig:crime_nrd} and \ref{fig:crime_prd}, by enforcing a lower bound on $\sum_i (\hy_i- y_i)$, positive and negative residual difference go to 0 very quickly---as expected. However, the trained model performs very poorly in terms of average utility of the disadvantaged group. 

\section{Conclusion}\label{sec:conclusion}
Our work makes an important contribution to the rapidly growing line of research on algorithmic fairness---by providing a unifying moral framework for understanding existing notions of fairness through philosophical interpretations and economic models of EOP. 
We showed that the choice between statistical parity, equality of odds, and predictive value parity can be mapped systematically to specific moral assumptions about what decision subjects morally deserve. Determining accountability features and effort-based utility is arguably outside the expertise of computer scientists, and has to be resolved through the appropriate process with input from stakeholders and domain experts. In any given application domain, reasonable people may disagree on what constitutes factors that people should be considered morally accountable for, and there will rarely be a consensus on the most suitable notion of fairness. This, however, does not imply that in a given context all existing notions of algorithmic fairness are equally acceptable from a moral standpoint.

\bibliographystyle{named}
\bibliography{FAT-Main}

\end{document}